\documentclass[conference]{IEEEtran}
\IEEEoverridecommandlockouts
\usepackage{cite}
\usepackage{amsmath,amssymb,amsfonts,amsthm}
\usepackage{algorithm,algorithmic}
\usepackage{graphicx}
\usepackage{float}
\usepackage{threeparttable}
\usepackage{array,makecell,booktabs,multirow,siunitx,bm}
\usepackage{textcomp}

\usepackage[table]{xcolor}
\usepackage{color}

\newtheorem{proposition}{Proposition}
\newtheorem{lemma}{Lemma}
\usepackage{hyperref}
\hypersetup{hidelinks,
	colorlinks=true,
	allcolors=black,
	pdfstartview=Fit,
	breaklinks=true}

\begin{document}

\title{SLaB: Sparse-Lowrank-Binary Decomposition for Efficient Large Language Models}
\author{
\IEEEauthorblockN{Ziwei Li, Yuang Ma, Yi Kang\IEEEauthorrefmark{1}\thanks{\IEEEauthorrefmark{1}Corresponding Author}\thanks{This work has been submitted to the IEEE for possible publication. Copyright may be transferred without notice, after which this version may no longer be accessible.}}
\IEEEauthorblockA{School of Microelectronics, University of Science and Technology of China, Hefei, Anhui, China\\
heinzt@mail.ustc.edu.cn, eya\_ma@mail.ustc.edu.cn, ykang@ustc.edu.cn}
}
\maketitle

\begin{abstract}
The rapid growth of large language models (LLMs) presents significant deployment challenges due to their massive computational and memory demands. While model compression, such as network pruning, offers potential solutions, most existing methods often fail to maintain good performance at high compression ratios. To address this, we propose SLaB, a novel framework that decomposes each linear layer weight into three complementary components: a sparse matrix, a low-rank matrix, and a binary matrix. SLaB eliminates the need for retraining and leverages activation-aware pruning scores to guide the decomposition process. Experiments on Llama-family models demonstrate that SLaB achieves state-of-the-art performance, reducing perplexity by up to $36\%$ compared to existing methods at $50\%$ compression and improving accuracy by up to $8.98\%$ over the baseline on zero-shot tasks.
\end{abstract}


\section{Introduction}
In recent years, large language models (LLMs) like GPT~\cite{GPT2, GPT3} and BERT~\cite{BERT} have driven significant progress in artificial intelligence. However, their large size and high computational demands make them difficult to deploy on resource-constrained devices.
Model compression~\cite{Modelcompression, Modelcompression2, Modelcompression3, Modelcompression4}, such as network pruning~\cite{OBD, OBS}, is an important solution that can effectively remove redundant parameters. However, pruning LLMs during training is computationally expensive. Many studies~\cite{SparseGPT, Wanda, SliceGPT, DSNoT} investigate one-shot pruning approaches for LLMs, while one-shot pruning causes significant accuracy degradation at high sparsity levels.
Meanwhile, low-rank decomposition like ASVD~\cite{ASVD} and SVD-LLM~\cite{SVD-LLM} offers a hardware-friendly approach. Combining sparsity with low-rank approximation may achieve higher accuracy at a high compression ratio. But we experimented with simply combining a sparsification method with a low-rank matrix, which yields poor results, as shown in Figure~\ref{fig:only_lowrank}.

\begin{figure}[ht]
\centering
\includegraphics[width=6cm]{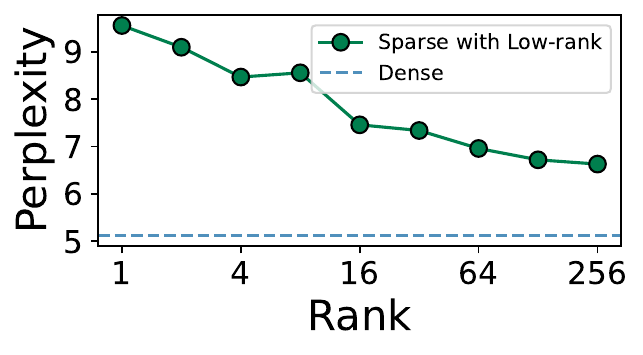}
\caption{Compression of the Llama-2 7B model~\cite{Llama2} using only low-rank and sparse matrices: perplexity comparison on the WikiText-2 dataset under different rank settings at a $50\%$ compression ratio.}\label{fig:only_lowrank}
\end{figure}

To solve these problems, we propose \textbf{Sparse-Lowrank-Binary Decomposition (SLaB)}, a novel pruning framework that combines low-rank and sparsity while using binary components as much as possible to simplify computations. More details are shown in Figure~\ref{fig:SLaB}. Our key insight is to decompose each linear layer weight $\mathbf{W}$ into three matrices:
\begin{equation}
\mathbf{W}=\mathbf{W}_{\text{S}}+\mathbf{W}_{\text{L}}\odot \mathbf{W}_{\text{B}},\label{eq:Linear_Layer_Decomposition}
\end{equation}
where $\odot$ denotes the Hadamard product, $\mathbf{W}_{\text{S}}$, $\mathbf{W}_{\text{B}}$ and $\mathbf{W}_{\text{L}}$ denote the sparse matrix, low-rank matrix, and binary matrix, respectively. $\mathbf{W}_{\text{B}}\in\{+1,-1\}$ has only binary elements which is hardware-friendly.

The fundamental idea of SLaB is to sparsify weight matrices in every linear layer of LLM while using an approach to compensate for the performance loss from sparsity. In SLaB, the original weight matrix can be replaced with a sparse matrix plus a low-rank matrix, which is used to compensate for loss. And finally, combining a binary matrix with a low-rank matrix can significantly reduce the required rank.

Our contributions are as follows: (1) we propose a novel decomposition method to decompose weight matrix in a linear layer into 3 matrices: a sparse matrix, a low-rank matrix, and a binary matrix to achieve weight compression; (2) we present an analysis to derive an algorithm for constructing these three matrices and (3) we evaluate SLaB on the widely adopted Llama-family models. Our results demonstrate that SLaB can retain strong performance without any training and outperform previous state-of-the-art pruning methods~\cite{SparseGPT, Wanda}.

\begin{figure*}[!t]
\centering
\includegraphics[width=13cm]{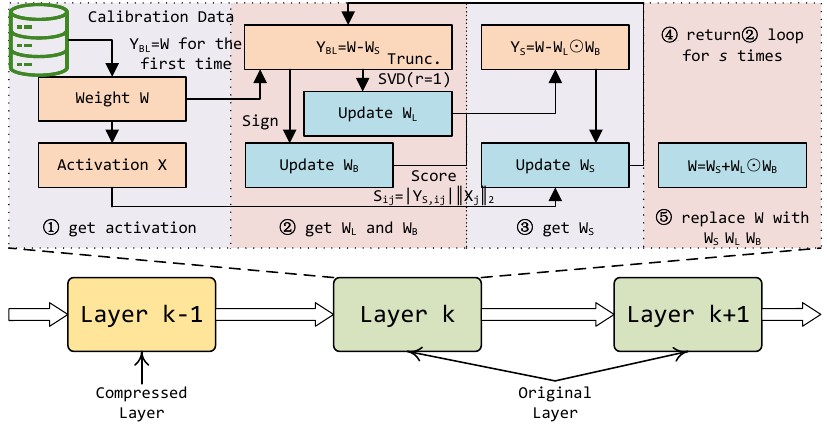}
\caption{Overview of the SLaB framework.}
\label{fig:SLaB}
\end{figure*}

\section{Methodology}
This section introduces SLaB. First, an optimization strategy is employed to transform the complex optimization problem into a single-matrix optimization task, with its optimization process being analyzed. Subsequently, the complete workflow is presented. Finally, key parameters are discussed.

Let the shape of weight matrix $\mathbf{W}$ be $(D_{\text{out}}, D_{\text{in}})$, and the shape of its corresponding activation matrix $\mathbf{X}$ be $(N \times L, D_{\text{in}})$, where $N$ and $L$ are the batch and sequence dimensions, respectively.

\subsection{Algorithm}
\subsubsection{Pruning Method}
Figure~\ref{fig:SLaB} provides an overview of the SLaB framework. At the weight level, to obtain $\mathbf{W}_{\text{S}}$, $\mathbf{W}_{\text{B}}$, and $\mathbf{W}_{\text{L}}$ in \eqref{eq:Linear_Layer_Decomposition}, we employ an alternating optimization strategy by fixing the other matrices during each update similar to the approach in \cite{RPCA, RPCA2, RPCA3}. Repeating this process for several iterations is sufficient for convergence. And at the layer level, we employ the layer-wise pruning approach~\cite{SparseGPT, Wanda}, which, in one-shot pruning, involves performing the following steps layer by layer: (1) forward propagation, (2) pruning, and (3) updating the layer's output after pruning.

\subsubsection{The Optimization of \texorpdfstring{$\mathbf{W}_{\text{S}}$}{WS}}
Given the integration of activation influences into the weight matrices via the Wanda~\cite{Wanda} technique, the optimization of $\mathbf{W}_{\text{S}}$ becomes straightforward. By fixing the remaining components and assuming $\mathbf{W}-\mathbf{W}_{\text{L}}\odot \mathbf{W}_{\text{B}}=\mathbf{Y}_{\text{S}}$, pruning is performed based on the magnitude of scoring matrix $\mathbf{S}$, where $\mathbf{S}_{ij}=|\mathbf{Y}_{\text{S},ij}|\cdot \Vert\mathbf{X}_j\Vert_2$. 

\subsubsection{The Optimization of \texorpdfstring{$\mathbf{W}_{\text{B}}$}{WB}}
Similarly, by fixing the remaining components and introducing the assumption $(\mathbf{W}-\mathbf{W}_{\text{S}})\oslash \mathbf{W}_{\text{L}}=\mathbf{Y}_{\text{B}}$, where $\oslash$ denotes element-wise division, the optimization objective can be reformulated as $
\mathop{\arg\min}_{\mathbf{W}_{\text{B}}}\Vert \mathbf{Y}_{\text{B}}-\mathbf{W}_{\text{B}}\Vert_2^2$. Furthermore, the following lemma and proposition can be proven:
\begin{lemma}\label{lemma:1}
If $w_0\leq w_1\leq\cdots\leq w_{n-1}, n\geq2$, $\tilde{a}$ and $\tilde{b}$ are the minimizers of $f(a,b)=\sum_{k=0}^{n-1}\min\left\{(w_k-a)^2,(w_k-b)^2\right\}$. Suppose $a\leq b$, there exists $t\in\{0, 1, \cdots, n-2\}$ such that
\begin{equation}
w_t\leq\frac{\tilde{a}+\tilde{b}}{2}\leq w_{t+1},\tilde{a}=\frac{\sum_{k=0}^{t}w_k}{t+1},\tilde{b}=\frac{\sum_{k=t+1}^{n-1}w_k}{n-t-1}.
\end{equation}
\end{lemma}
\begin{proof}
Clearly, $\tilde{a}$ and $\tilde{b}$ exist; we proceed by contradiction. Furthermore, the conclusion is trivial if either all $w_k$ are identical or $\tilde{a}=\tilde{b}$. Moreover, the minimum can only be attained when $a$ and $b$ take values within the interval $[w_0,w_{n-1}]$. In the following, we assume $a,b\in[w_0,w_{n-1}]$. Without loss of generality, suppose $\tilde{a}<\sum_{k=0}^{t} w_k/(t+1)\leq w_{n-1}$ which implies that there exists $t$ satisfying $w_t \leq (\tilde{a} + \tilde{b})/2 < w_{t+1}$. So that $f(a,b)=\sum_{k=0}^{t}(w_k-a)^2+\sum_{k=t+1}^{n-1}(w_k-b)^2$.
Then for $a \in[2w_t - \tilde{b}, 2w_{t+1} - \tilde{b}]$, since $f(a,b)$ is a quadratic function in $a$, we have
\begin{equation}
f(\tilde{a},\tilde{b}) > f\left(\min\left\{2w_{t+1}-\tilde{b},\frac{\sum_{k=0}^{t} w_k}{t+1}\right\}, \tilde{b}\right).
\end{equation}
This contradicts the assumption. Similar arguments apply to the remaining cases, thus completing the proof of the lemma.
\end{proof}
\begin{proposition}\label{prop:1}
If a random matrix $\mathbf{X}\sim F$ possesses a symmetric probability density function $\phi(x)=\phi(-x)$, and the matrix $\mathbf{W}_{\text{B}}\in\{a,b\}$ is a binary-valued matrix, then a suboptimal solution to the optimization problem
\begin{equation}
\mathop{\arg\min}_{\mathbf{W}_{\text{B}}} \Vert \mathbf{X} - \mathbf{W}_{\text{B}} \Vert_2^2\label{eq:Prop2_Condition}
\end{equation}
satisfies $a+b=0$.
\end{proposition}
\begin{proof}
Based on the conclusion of Lemma~\ref{lemma:1}, a suboptimal solution to~\eqref{eq:Prop2_Condition} satisfies:
\begin{equation}
w_{t}\leq\frac{\mathbb{E}[\mathbf{X}|\mathbf{X}\leq w_{t}]+\mathbb{E}[\mathbf{X}|\mathbf{X}\geq w_{t+1}]}{2}\leq w_{t+1}.\label{eq:proof1}
\end{equation}
Assuming $w_t \approx w_{t+1} = s$, we obtain that $\mathbb{E}[\mathbf{X}|\mathbf{X}\leq s]+\mathbb{E}[\mathbf{X}|\mathbf{X}\geq s]=2s$. Clearly, when $s = 0$, since $\phi(x) = \phi(-x)$, \eqref{eq:proof1} holds. Therefore, $s = 0$ is a suboptimal solution, which implies $a + b = 0$ is a suboptimal solution.
\end{proof}
Noting that our optimization process does not favor positive or negative values—together with the common assumption that the original linear layer weights $\mathbf{W}$ follow a zero-mean distribution, we apply Proposition~\ref{prop:1} and find that using a binary matrix $\mathbf{W}_{\text{B}}$ with values from the set $\{a, -a\}$ constitutes a suboptimal solution. Given the Hadamard product relationship between $\mathbf{W}_{\text{B}}$ and $\mathbf{W}_{\text{L}}$, the scaling factors of $\mathbf{W}_{\text{B}}$ can be absorbed into $\mathbf{W}_{\text{L}}$. Therefore, $\mathbf{W}_{\text{B}}$ is constrained to the binary set $\{+1, -1\}$.

\subsubsection{The Optimization of \texorpdfstring{$\mathbf{W}_{\text{L}}$}{WL}}
Similarly, with the assumption that $(\mathbf{W} - \mathbf{W}_{\text{S}}) \oslash \mathbf{W}_{\text{B}} = \mathbf{Y}_{\text{L}}$, the optimization problem reduces to $\mathop{\arg\min}_{\mathbf{W}_{\text{L}}} \Vert \mathbf{Y}_{\text{L}} - \mathbf{W}_{\text{L}} \Vert_2^2$. Guided by the Eckart–Young theorem~\cite{EY-Theorem}, the optimal solution with a specified rank constraint is obtained by truncated singular value decomposition (SVD).

\subsubsection{The Joint Optimization of \texorpdfstring{$\mathbf{W}_{\text{B}}$}{WB} and \texorpdfstring{$\mathbf{W}_{\text{L}}$}{WL}}\label{sec:WB_WL_analysis}
Let $\mathbf{W} - \mathbf{W}_{\text{S}} = \mathbf{Y}_{\text{BL}}$. To alternately optimize $\mathbf{W}_{\text{B}}$ and $\mathbf{W}_{\text{L}}$, we naturally adopt the following straightforward initialization:
\begin{equation}
\begin{aligned}
\mathbf{W}_{\text{B}} &= \text{sign}(\mathbf{Y}_{\text{BL}}),\\
\mathbf{W}_{\text{L}} &= \mathbf{U}\mathbf{V}^{\mathrm{T}}\;\text{with}\;\mathbf{U}=\sqrt{\sigma_0}\mathbf{u}_0,\mathbf{V}=\sqrt{\sigma_0}\mathbf{v}_0,\\
|\mathbf{Y}_{\text{BL}}| &= \sum_{k=0}^{\min\{D_{\text{out}}, D_{\text{in}}\}}\sigma_k\mathbf{u}_k\mathbf{v}_k^{\mathrm{T}},\label{eq:WB_WL_Init}
\end{aligned}
\end{equation}
where $\text{sign}(\cdot)$ denotes a function that judges the input as positive or negative, non-negative numbers are denoted as 1 while negative numbers are denoted as 0, $\mathbf{U}$ and $\mathbf{V}$ are computed from $\mathbf{Y}_{\text{BL}}$ using a rank-1 truncated SVD.
\begin{proposition}\label{prop:2}
If $\mathbf{U}\in\mathbb{R}^{(D_{\text{out}},1)}$ and $\mathbf{V}\in\mathbb{R}^{(D_{\text{in}},1)}$ are the rank-1 truncated SVD results of $|\mathbf{Y}_{\text{BL}}|$, satisfying $\mathbf{W}_{\text{L}}=\mathbf{U}\mathbf{V}^{\mathrm{T}}$. Then, every element of $\mathbf{W}_{\text{L}}$ is non-negative, i.e. $\lvert \mathbf{W}_{\text{L}}\rvert = \mathbf{W}_{\text{L}}$.
\end{proposition}
\begin{proof}
Consider the Frobenius norm of the difference between each of these matrices and $|\mathbf{Y}_{\text{BL}}|$. Clearly, both $\mathbf{W}_{\text{L}}$ and $|\mathbf{W}_{\text{L}}|$ are rank-1 matrices, so we have (due to the Eckart-Young theorem):
\begin{equation}
\Vert \mathbf{W}_{\text{L}}-|\mathbf{Y}_{\text{BL}}|\Vert_F^2\leq\Vert |\mathbf{W}_{\text{L}}|-|\mathbf{Y}_{\text{BL}}|\Vert_F^2,
\end{equation}
where equality holds if and only if $\mathbf{W}_{\text{L}} = |\mathbf{W}_{\text{L}}|$.
However, by the triangle inequality,
\begin{equation}
\Vert \mathbf{W}_{\text{L}}-|\mathbf{Y}_{\text{BL}}|\Vert_F^2\geq\Vert |\mathbf{W}_{\text{L}}|-|\mathbf{Y}_{\text{BL}}|\Vert_F^2.
\end{equation}
So we have $\mathbf{W}_{\text{L}}=|\mathbf{W}_{\text{L}}|$ and the proposition is proven.
\end{proof}
When using the initialization in \eqref{eq:WB_WL_Init}, if $\mathbf{W}_{\text{B}}$ is fixed as $\text{sign}(\mathbf{Y}_{\text{BL}})$, the optimal solution for $\mathbf{W}_{\text{L}}$ is given by $\mathbf{U}\mathbf{V}^{\mathrm{T}}$. Conversely, if $\mathbf{W}_{\text{L}}$ is fixed, Proposition~\ref{prop:2} guarantees its non-negativity, and thus the optimal choice for $\mathbf{W}_{\text{B}}$ is $\text{sign}(\mathbf{Y}_{\text{BL}})$.

So \eqref{eq:WB_WL_Init} yields a suboptimal solution for both $\mathbf{W}_{\text{B}}$ and $\mathbf{W}_{\text{L}}$. Adopting the initialization from \eqref{eq:WB_WL_Init} significantly reduces the computational overhead of alternating optimization while maintaining efficacy. And the rank-1 setting will be discussed in the Section~\ref{sec:rank_setting}.

\subsection{Parameters}
\subsubsection{Sparsity}
Let the original matrix $\mathbf{W}$, along with $\mathbf{W}_{\text{S}}$ and $\mathbf{W}_{\text{L}}$, be $b$-bit wide data (e.g. $b=16$ for FP16). Suppose $\mathbf{W}_{\text{S}}$ contains $k$ non-zero elements, $\mathbf{W}_{\text{B}}$ is a binarized matrix where each element is stored using 1 bit, and $\mathbf{W}_{\text{L}}$ is a rank-1 matrix expressible as the product of two vectors with shapes $(D_{\text{out}}, 1)$ and $(1, D_{\text{in}})$, respectively. Thus, the compression ratio ($\text{CR}$) of SLaB is:  
\begin{equation}\label{eq:CR_Def}
\text{CR}=1-\frac{(\overbrace{bk}^{\mathbf{W}_{\text{S}}}+\overbrace{D_{\text{out}}D_{\text{in}}}^{\mathbf{W}_{\text{B}}}+\overbrace{b(D_{\text{out}} + D_{\text{in}})}^{\mathbf{W}_{\text{L}}})}{bD_{\text{out}}D_{\text{in}}}.
\end{equation}
Therefore, the percentage of non-zero elements retained in the sparse matrix $\mathbf{W}_{\text{S}}$ relative to the total elements can be derived:
\begin{equation}
\frac{k}{D_{\text{out}}D_{\text{in}}} = 1-\text{CR}-\frac{1}{b}-\frac{1}{D_{\text{out}}}-\frac{1}{D_{\text{in}}}.
\end{equation}

\subsubsection{Comparison Group Size}
Following~\cite{SparseGPT, Wanda}, we use comparison groups of size $(1, D_{\text{in}})$ (i.e., the number of elements in each group $N_g = D_{\text{in}}$), within which we prune by comparing the score matrices. This results in $\lfloor kN_g/(D_{\text{out}}D_{\text{in}})\rfloor=\lfloor k/D_{\text{out}}\rfloor$ non-zero elements retained per group. For the semi-structured pruning method \cite{NVsparse} like 2:4 and 4:8, we first apply semi-structured pruning and then perform group-wise pruning to reach the target sparsity.

\subsubsection{Rank and Iterations}\label{sec:rank_setting}
As shown in Figure~\ref{fig:rank_setting}, when all layers are pruned at a $50\%$ compression ratio, increasing the rank from 0 (corresponding to the Wanda method) to 1 significantly reduces the average Frobenius norm difference between the compressed and original weight. However, as the rank continues to increase, the reduction becomes much more gradual. Therefore, we adopt the rank-1 setting for simplicity.

\begin{figure}[!t]
\centering
\includegraphics[width=8cm]{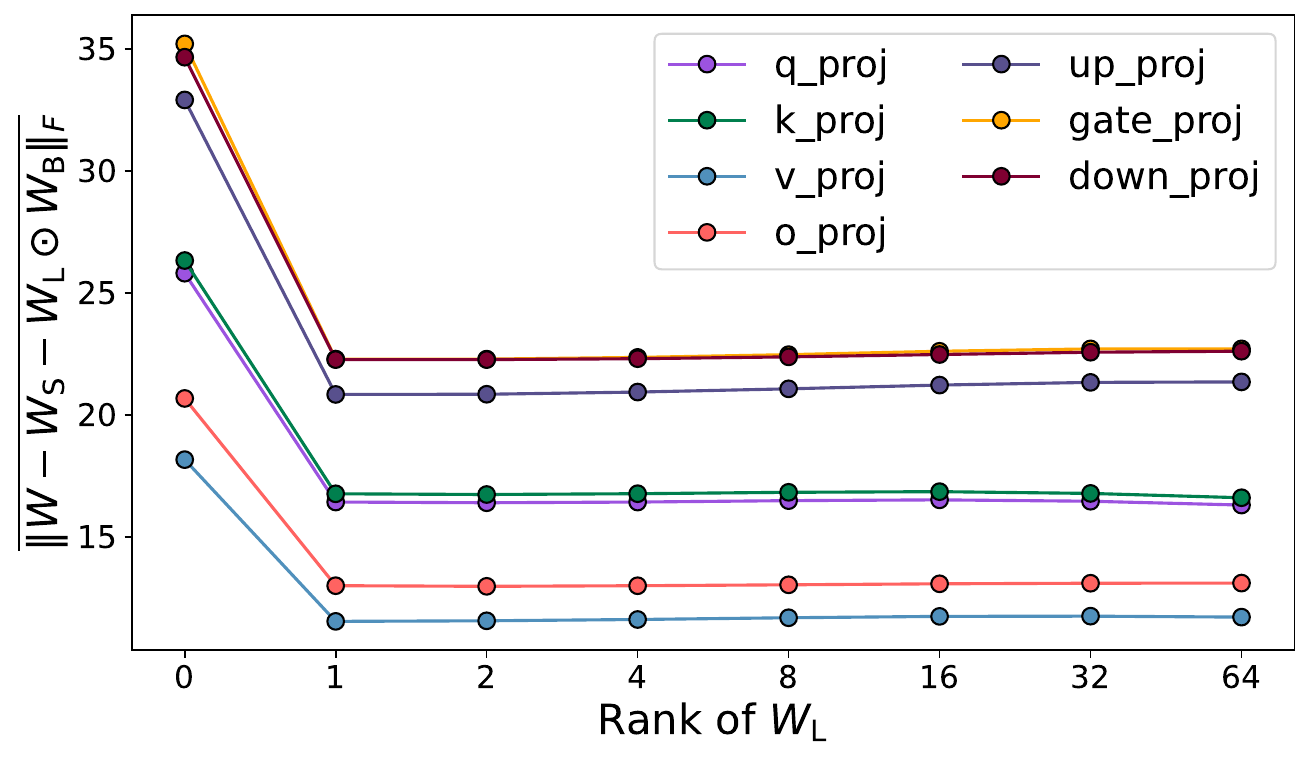}
\caption{Variation of the average Frobenius norm difference between compressed and original layers with respect to rank. Experiments are conducted on the Llama-2 7B model~\cite{Llama2} with a $50\%$ compression ratio.}\label{fig:rank_setting}
\end{figure}
\begin{algorithm}[!t]
\renewcommand{\algorithmicrequire}{\textbf{Input:}}
\renewcommand{\algorithmicensure}{\textbf{Output:}}
\caption{The SLaB algorithm.}
\label{alg:SLaB}
\begin{algorithmic}[1]
\REQUIRE Weight $\mathbf{W}\in\mathbb{R}^{(D_{\text{out}}, D_{\text{in}})}$, activation obtained through the forward propagation of the preceding layer $\mathbf{X}\in\mathbb{R}^{(D_{\text{out}}, D_{\text{in}})}$, number of iterations $s$, compression ratio $\text{CR}$, the bit-width of the output matrix (excluding the binary matrix) $b$
\ENSURE Sparse weight $\mathbf{W}_{\text{S}}\in\mathbb{R}^{(D_{\text{out}}, D_{\text{in}})}$, matrix $\mathbf{U}\in\mathbb{R}^{(D_{\text{out}}, 1)}, \mathbf{V}\in\mathbb{R}^{(D_{\text{in}}, 1)}$, binary matrix $\mathbf{W}_{\text{B}}\in\{1,-1\}^{(D_{\text{out}}, D_{\text{in}})}$ 
\STATE Init $\mathbf{W}_{\text{S}} \gets \mathbf{0}$, $\mathbf{U} \gets \mathbf{0}$, $\mathbf{V} \gets \mathbf{0}$, $\mathbf{W}_{\text{B}} \gets \mathbf{0}$
\STATE $\text{sparsity} \gets 1-\text{CR}-\frac{1}{b}-\frac{1}{D_{\text{out}}}-\frac{1}{\text{in}}$
\STATE $\mathbf{S}_{\text{X}} \gets \text{diag}\left(\sqrt{\mathbf{X}^{\mathrm{T}} \mathbf{X}}\right)\in\mathbb{R}^{(1, D_{\text{in}})}$
\FOR{step $t = 1$ \textbf{to} $s$}
\STATE $\mathbf{W}_{\text{B}} \gets \text{sign}(\mathbf{W}-\mathbf{W}_{\text{S}})$
\STATE $\mathbf{U},\mathbf{V} \gets \sqrt{\sigma_0}\mathbf{u}_0, \sqrt{\sigma_0}\mathbf{v}_0$\\
$\text{with}\;|\mathbf{W}-\mathbf{W}_\text{S}|=\sum_{k=0}^{\min\{D_{\text{out}}, D_{\text{in}}\}}\sigma_k\mathbf{u}_k\mathbf{v}_k^{\mathrm{T}}$
\STATE $\mathbf{S} \gets \left|\mathbf{W}-\mathbf{U}\mathbf{V}^{\mathrm{T}}\odot \mathbf{W}_{\text{B}}\right|\odot \mathbf{S}_{\text{X}}$
\STATE $\mathbf{W}_{\text{S}} \gets \text{HardThreshold}(\mathbf{S}, \text{sparsity})\oslash \mathbf{S}_{\text{X}}$
\ENDFOR
\RETURN $\mathbf{W}_{\text{S}}, \mathbf{U}, \mathbf{V}, \mathbf{W}_{\text{B}}$
\end{algorithmic}
\end{algorithm}

In addition to the parameters mentioned above, the number of iterations $s$ in the alternating optimization algorithm also affects the outcome. We will conduct hyperparameter exploration on this in the experimental section.

\subsection{Full Algorithm Pseudocode}
Algorithm~\ref{alg:SLaB} presents the complete pruning process of a linear layer using SLaB. The HardThreshold function prunes a given matrix by magnitude according to a specified CR. Comparison group size is not explicitly shown in Algorithm~\ref{alg:SLaB}.

\section{Experiments}
\subsection{Experiments Setup}
\subsubsection{Models and Baselines}
We perform pruning and analysis on the Llama-2 7B~\cite{Llama2}, Llama-3 8B, and Llama-3.2 1B~\cite{Llama3} models using PyTorch and the Transformers library provided by Hugging Face. The SLaB is a one-shot pruning method, so that we compare results against two SOTA, also one-shot approaches: SparseGPT~\cite{SparseGPT} and Wanda~\cite{Wanda}.

\subsubsection{Calibration Data}
Consistent with the approach in SparseGPT, we construct the calibration dataset using 128 sequences of length 2048, which are randomly sampled from the first shard of the C4 training set~\cite{C4}. To ensure a fair comparison, all pruning algorithms employ this identical calibration dataset (where applicable).
\subsubsection{Evaluation}
We evaluate the performance of the pruned models across both zero-shot tasks and language modeling capabilities. We select seven tasks include ARC-C, ARC-E~\cite{Arc}, BoolQ~\cite{Boolq}, HellaSwag~\cite{HellaSwag}, PIQA~\cite{PIQA}, RTE~\cite{GLUE}, and WinoGrande~\cite{WinoGrande} from the LM-Eval-Harness~\cite{LM-Eval-Harness} for zero-shot evaluation. Language modeling capability is quantified by measuring perplexity on the WikiText-2 dataset~\cite{WikiText-2}.

\subsubsection{Parameters}
Like SparseGPT and Wanda, we focus our pruning efforts on all linear layers, excluding the first embedding layers and the final classification head. The pruning is carried out in different categories: unstructured/structured sparsity and CR (Compression Ratio). For unstructured sparsity, we evaluate performance with $\text{CR}=50\%,60\%,70\%,80\%$; For structured sparsity, we evaluate two specific sparsity patterns: 2:4 and 4:8, both configured with $\text{CR}=50\%$. Regarding the comparison group size, we set it to $(1, D_{\text{in}})$. The number of iterations for SLaB is set to 20.
\begin{table}[!t]
\centering
\caption{Comparison of perplexity on WikiText-2 and average zero-shot accuracies ($\%$) at different compression rates.}\label{tab:main_result}
\setlength{\tabcolsep}{1.5mm}
\scriptsize
\begin{threeparttable}
\begin{tabular}{lw{c}{1.3cm}w{c}{0.6cm}w{c}{0.6cm}w{c}{0.6cm}w{c}{0.6cm}w{c}{0.6cm}w{c}{0.6cm}}
\toprule
\multirow{2}{*}{Method} & \multirow{2}{*}{Sparsity($\text{CR}$)} & \multicolumn{2}{c}{Llama-3.2 1B} & \multicolumn{2}{c}{Llama-2 7B} & \multicolumn{2}{c}{Llama-3 8B} \\
~ & ~ & ppl\tnote{2}\,$\downarrow$ & acc\tnote{2}\,$\uparrow$ & ppl$\downarrow$ & acc$\uparrow$ & ppl$\downarrow$ & acc$\uparrow$ \\
\midrule
Dense&$0\%$&$9.06$&$58.3$&$5.12$&$68.4$&$5.75$&$72.9$\\
\midrule
SparseGPT&US\tnote{1} ($50\%$)&$18.09$&$51.2$&$6.52$&$63.9$&$8.73$&$66.9$\\
Wanda&US ($50\%$)&$21.37$&$48.5$&$6.45$&$64.0$&$9.15$&$65.4$\\
\rowcolor{gray!15}SLaB&US ($50\%$)&$\pmb{11.57}$&$\pmb{55.8}$&$\pmb{5.49}$&$\pmb{66.2}$&$\pmb{6.67}$&$\pmb{70.7}$\\
\midrule
SparseGPT&4:8 ($50\%$)&$21.64$&$48.0$&$7.94$&$61.6$&$11.33$&$61.3$\\
Wanda&4:8 ($50\%$)&$40.47$&$44.3$&$8.03$&$60.8$&$13.41$&$65.4$\\
\rowcolor{gray!15}SLaB&4:8 ($50\%$)&$\pmb{12.43}$&$\pmb{52.5}$&$\pmb{5.61}$&$\pmb{65.3}$&$\pmb{6.93}$&$\pmb{70.1}$\\
\midrule
SparseGPT&2:4 ($50\%$)&$30.14$&$46.5$&$10.37$&$58.9$&$14.81$&$57.4$\\
Wanda&2:4 ($50\%$)&$79.06$&$41.3$&$11.40$&$55.6$&$22.52$&$52.2$\\
\rowcolor{gray!15}SLaB&2:4 ($50\%$)&$\pmb{14.02}$&$\pmb{52.5}$&$\pmb{5.77}$&$\pmb{65.3}$&$\pmb{7.32}$&$\pmb{70.1}$\\
\midrule
SparseGPT&US ($60\%$)&$44.43$&$45.2$&$9.52$&$58.7$&$14.45$&$59.0$\\
Wanda&US ($60\%$)&$71.92$&$41.5$&$10.01$&$57.1$&$21.37$&$53.3$\\
\rowcolor{gray!15}SLaB&US ($60\%$)&$\pmb{14.73}$&$\pmb{51.7}$&$\pmb{5.85}$&$\pmb{65.0}$&$\pmb{7.54}$&$\pmb{69.2}$\\
\midrule
SparseGPT&US ($70\%$)&$132.62$&$40.8$&$24.84$&$47.6$&$38.52$&$46.9$\\
Wanda&US ($70\%$)&$412.56$&$37.4$&$71.58$&$39.9$&$107.00$&$39.3$\\
\rowcolor{gray!15}SLaB&US ($70\%$)&$\pmb{26.06}$&$\pmb{46.6}$&$\pmb{6.88}$&$\pmb{62.2}$&$\pmb{10.35}$&$\pmb{64.3}$\\
\midrule
SparseGPT&US ($80\%$)&$471.50$&$38.5$&$111.27$&$37.6$&$183.01$&$39.5$\\
Wanda&US ($80\%$)&$1.34\mathrm{e}{4}$&$36.6$&$3.60\mathrm{e}{3}$&$36.1$&$1.25\mathrm{e}{3}$&$37.2$\\
\rowcolor{gray!15}SLaB&US ($80\%$)&$\pmb{113.77}$&$\pmb{41.5}$&$\pmb{13.63}$&$\pmb{53.4}$&$\pmb{27.02}$&$\pmb{49.3}$\\
\bottomrule
\end{tabular}
\begin{tablenotes}
\item[1] US refers to unstructured sparsity.
\item[2] ppl refers to perplexity (lower is better) and acc refers to accuracy (higher is better).
\end{tablenotes}
\end{threeparttable}
\end{table}
\subsection{Results}
Table~\ref{tab:main_result} presents perplexity on the WikiText-2 dataset and average accuracies on zero-shot tasks for SLaB and baselines. Using $50\%$ compression ratio for the Llama-3.2 1B model, SLaB reduces perplexity by $36.04\%$ and improves accuracy by $8.98\%$ compared to the best-performing baseline.

\subsection{Hyperparameter Exploration}
Table~\ref{tab:hyper_parameter} presents the impact of the comparison group size and the number of iterations on the results.
\begin{table}[!t]
\caption{Hyperparameter exploration with Llama-2 7B ($\text{CR}=50\%$) model.}\label{tab:hyper_parameter}
\centering
\setlength{\tabcolsep}{2mm}
\scriptsize
\begin{tabular}{cccccc}
\toprule
Comparison&\multirow{2}{*}{$\left(1, \displaystyle\frac{D_{\text{in}}}{32}\right)$}&\multirow{2}{*}{$\left(1, \displaystyle\frac{D_{\text{in}}}{16}\right)$}&\multirow{2}{*}{$(1, D_{\text{in}})$}&\multirow{2}{*}{$(16, D_{\text{in}})$}&\multirow{2}{*}{$(32, D_{\text{in}})$}\\
Group& ~ & ~ & ~ & ~ & ~ \\
\midrule
ppl$\downarrow$&$5.516$&$5.491$&$5.493$&$5.544$&$5.546$\\
acc$\uparrow$($\%$)&$65.6$&$65.8$&$66.2$&$65.9$&$66.2$\\
\bottomrule
\toprule
Iterations&1&10&20&30&40\\
\midrule
ppl($\downarrow$)&$5.678$&$5.531$&$5.493$&$5.480$&$5.477$\\
\bottomrule
\end{tabular}
\end{table}

\subsection{Ablation Study}
As shown in Table~\ref{tab:ablation}, the incorporation of both the $\mathbf{W}_{\text{L}}$ and $\mathbf{W}_{\text{B}}$ components has a compensatory effect on $\mathbf{W}_{\text{S}}$, thereby enhancing the model's accuracy.

\begin{table}[!t]
\centering
\caption{Ablation study with Llama-2 7B (2:4, $\text{CR}=50\%$) model.}\label{tab:ablation}
\setlength{\tabcolsep}{2mm}
\scriptsize
\begin{threeparttable}
\begin{tabular}{cccccc}
\toprule
Accuracy ($\%$)&ARC-C&ARC-E&RTE&WinoGrande&Avg\\
\midrule
$\mathbf{W}_{\text{S}}$&$32.0$&$58.2$&$53.1$&$62.0$&$49.8$\\
$\mathbf{W}_{\text{S}}+\mathbf{W}_{\text{L}}(r=16)$&$33.4$&$59.2$&$53.8$&$64.8$&$51.2$\\
$\mathbf{W}_{\text{S}}+\text{factor}\tnote{*}\odot\mathbf{W}_{\text{B}}$&$42.1$&$71.8$&$55.6$&$68.2$&$58.2$\\
$\mathbf{W}_{\text{S}}+\mathbf{W}_{\text{L}}\odot\mathbf{W}_{\text{B}}$&$43.2$&$71.3$&$57.0$&$68.6$&$58.9$\\
\bottomrule
\end{tabular}
\begin{tablenotes}
\item[*] factor refers to the quantization factor vector.
\end{tablenotes}
\end{threeparttable}
\end{table}

\section{Conclusion}
In this work, we propose a novel method called SLaB for compressing LLMs. SLaB replaces weights with a combination of sparse, binary, and low-rank components, achieving better performance without any training. This work provides a new perspective for future research on combining sparsity with other compression techniques in a complementary manner.

\newpage

\bibliographystyle{IEEEtran}
\bibliography{ref}

@TechReport{GPT2,
  author      = {Radford, Alec and Wu, Jeffrey and Child, Rewon and Luan, David and Amodei, Dario and Sutskever, Ilya and others},
  institution = {OpenAI},
  title       = {Language models are unsupervised multitask learners},
  year        = {2019},
  number      = {8},
  journal     = {OpenAI blog},
  pages       = {9},
  volume      = {1},
}

@Article{GPT3,
  author  = {Brown, Tom and Mann, Benjamin and Ryder, Nick and Subbiah, Melanie and Kaplan, Jared D and Dhariwal, Prafulla and Neelakantan, Arvind and Shyam, Pranav and Sastry, Girish and Askell, Amanda and others},
  journal = {Advances in neural information processing systems},
  title   = {Language models are few-shot learners},
  year    = {2020},
  pages   = {1877--1901},
  volume  = {33},
}

@InProceedings{BERT,
  author    = {Devlin, Jacob and Chang, Ming-Wei and Lee, Kenton and Toutanova, Kristina},
  booktitle = {Proceedings of the 2019 conference of the North American chapter of the association for computational linguistics: human language technologies, volume 1 (long and short papers)},
  title     = {Bert: Pre-training of deep bidirectional transformers for language understanding},
  year      = {2019},
  pages     = {4171--4186},
}

@Article{SliceGPT,
  author  = {Ashkboos, Saleh and Croci, Maximilian L and Nascimento, Marcelo Gennari do and Hoefler, Torsten and Hensman, James},
  journal = {arXiv preprint arXiv:2401.15024},
  title   = {Slicegpt: Compress large language models by deleting rows and columns},
  year    = {2024},
}

@Article{Wanda,
  author  = {Sun, Mingjie and Liu, Zhuang and Bair, Anna and Kolter, J Zico},
  journal = {arXiv preprint arXiv:2306.11695},
  title   = {A simple and effective pruning approach for large language models},
  year    = {2023},
}

@InProceedings{SparseGPT,
  author       = {Frantar, Elias and Alistarh, Dan},
  booktitle    = {International Conference on Machine Learning},
  title        = {Sparsegpt: Massive language models can be accurately pruned in one-shot},
  year         = {2023},
  organization = {PMLR},
  pages        = {10323--10337},
}

@Article{Modelcompression,
  author  = {Han, Song and Mao, Huizi and Dally, William J},
  journal = {arXiv preprint arXiv:1510.00149},
  title   = {Deep compression: Compressing deep neural networks with pruning, trained quantization and huffman coding},
  year    = {2015},
}

@Article{Modelcompression2,
  author  = {Hinton, Geoffrey and Vinyals, Oriol and Dean, Jeff},
  journal = {arXiv preprint arXiv:1503.02531},
  title   = {Distilling the knowledge in a neural network},
  year    = {2015},
}

@Article{Modelcompression3,
  author  = {Frankle, Jonathan and Carbin, Michael},
  journal = {arXiv preprint arXiv:1803.03635},
  title   = {The lottery ticket hypothesis: Finding sparse, trainable neural networks},
  year    = {2018},
}

@Article{Modelcompression4,
  author  = {Fan, Angela and Stock, Pierre and Graham, Benjamin and Grave, Edouard and Gribonval, R{\'e}mi and Jegou, Herve and Joulin, Armand},
  journal = {arXiv preprint arXiv:2004.07320},
  title   = {Training with quantization noise for extreme model compression},
  year    = {2020},
}

@Article{OBD,
  author  = {LeCun, Yann and Denker, John and Solla, Sara},
  journal = {Advances in neural information processing systems},
  title   = {Optimal brain damage},
  year    = {1989},
  volume  = {2},
}

@Article{OBS,
  author  = {Hassibi, Babak and Stork, David},
  journal = {Advances in neural information processing systems},
  title   = {Second order derivatives for network pruning: Optimal brain surgeon},
  year    = {1992},
  volume  = {5},
}

@Article{ASVD,
  author  = {Yuan, Zhihang and Shang, Yuzhang and Song, Yue and Wu, Qiang and Yan, Yan and Sun, Guangyu},
  journal = {arXiv preprint arXiv:2312.05821},
  title   = {Asvd: Activation-aware singular value decomposition for compressing large language models},
  year    = {2023},
}

@Article{SVD-LLM,
  author  = {Wang, Xin and Zheng, Yu and Wan, Zhongwei and Zhang, Mi},
  journal = {arXiv preprint arXiv:2403.07378},
  title   = {Svd-llm: Truncation-aware singular value decomposition for large language model compression},
  year    = {2024},
}

@Article{DSNoT,
  author  = {Zhang, Yuxin and Zhao, Lirui and Lin, Mingbao and Sun, Yunyun and Yao, Yiwu and Han, Xingjia and Tanner, Jared and Liu, Shiwei and Ji, Rongrong},
  journal = {arXiv preprint arXiv:2310.08915},
  title   = {Dynamic sparse no training: Training-free fine-tuning for sparse llms},
  year    = {2023},
}

@Article{EY-Theorem,
  author    = {Eckart, Carl and Young, Gale},
  journal   = {Psychometrika},
  title     = {The approximation of one matrix by another of lower rank},
  year      = {1936},
  number    = {3},
  pages     = {211--218},
  volume    = {1},
  publisher = {Springer-Verlag},
}

@InProceedings{RPCA3,
  author    = {Zhou, Tianyi and Tao, Dacheng},
  booktitle = {Proceedings of the 28th International Conference on Machine Learning, ICML 2011},
  title     = {Godec: Randomized low-rank \& sparse matrix decomposition in noisy case},
  year      = {2011},
}

@Article{RPCA2,
  author    = {Cand{\`e}s, Emmanuel J and Li, Xiaodong and Ma, Yi and Wright, John},
  journal   = {Journal of the ACM (JACM)},
  title     = {Robust principal component analysis?},
  year      = {2011},
  number    = {3},
  pages     = {1--37},
  volume    = {58},
  publisher = {ACM New York, NY, USA},
}

@Article{RPCA,
  author  = {Wright, John and Ganesh, Arvind and Rao, Shankar and Peng, Yigang and Ma, Yi},
  journal = {Advances in neural information processing systems},
  title   = {Robust principal component analysis: Exact recovery of corrupted low-rank matrices via convex optimization},
  year    = {2009},
  volume  = {22},
}

@Article{C4,
  author  = {Raffel, Colin and Shazeer, Noam and Roberts, Adam and Lee, Katherine and Narang, Sharan and Matena, Michael and Zhou, Yanqi and Li, Wei and Liu, Peter J},
  journal = {Journal of machine learning research},
  title   = {Exploring the limits of transfer learning with a unified text-to-text transformer},
  year    = {2020},
  number  = {140},
  pages   = {1--67},
  volume  = {21},
}

@Software{LM-Eval-Harness,
  author    = {Lintang Sutawika and Hailey Schoelkopf and Leo Gao and Baber Abbasi and Stella Biderman and Jonathan Tow and ben fattori and Charles Lovering and farzanehnakhaee70 and Jason Phang and Anish Thite and Fazz and Thomas Wang and Niklas Muennighoff and Aflah and sdtblck and nopperl and gakada and tttyuntian and researcher2 and Julen Etxaniz and Chris and Hanwool Albert Lee and Khalid and Zdeněk Kasner and LSinev and KonradSzafer and Jeffrey Hsu and Anjor Kanekar and Pawan Sasanka Ammanamanchi},
  month     = sep,
  publisher = {Zenodo},
  title     = {EleutherAI/lm-evaluation-harness: v0.4.4},
  url       = {https://doi.org/10.5281/zenodo.13694023},
  version   = {v0.4.4},
  year      = {2024},
}

@Article{Wikitext-2,
  author  = {Merity, Stephen and Xiong, Caiming and Bradbury, James and Socher, Richard},
  journal = {arXiv preprint arXiv:1609.07843},
  title   = {Pointer sentinel mixture models},
  year    = {2016},
}

@Article{Arc,
  author  = {Clark, Peter and Cowhey, Isaac and Etzioni, Oren and Khot, Tushar and Sabharwal, Ashish and Schoenick, Carissa and Tafjord, Oyvind},
  journal = {arXiv preprint arXiv:1803.05457},
  title   = {Think you have solved question answering? try arc, the ai2 reasoning challenge},
  year    = {2018},
}

@Article{Boolq,
  author  = {Clark, Christopher and Lee, Kenton and Chang, Ming-Wei and Kwiatkowski, Tom and Collins, Michael and Toutanova, Kristina},
  journal = {arXiv preprint arXiv:1905.10044},
  title   = {Boolq: Exploring the surprising difficulty of natural yes/no questions},
  year    = {2019},
}

@Article{HellaSwag,
  author  = {Zellers, Rowan and Holtzman, Ari and Bisk, Yonatan and Farhadi, Ali and Choi, Yejin},
  journal = {arXiv preprint arXiv:1905.07830},
  title   = {Hellaswag: Can a machine really finish your sentence?},
  year    = {2019},
}

@InProceedings{PIQA,
  author    = {Bisk, Yonatan and Zellers, Rowan and Gao, Jianfeng and Choi, Yejin and others},
  booktitle = {Proceedings of the AAAI conference on artificial intelligence},
  title     = {Piqa: Reasoning about physical commonsense in natural language},
  year      = {2020},
  number    = {05},
  pages     = {7432--7439},
  volume    = {34},
}

@InProceedings{GLUE,
  author    = {Wang, Alex and Singh, Amanpreet and Michael, Julian and Hill, Felix and Levy, Omer and Bowman, Samuel},
  booktitle = {Proceedings of the 2018 EMNLP workshop BlackboxNLP: Analyzing and interpreting neural networks for NLP},
  title     = {GLUE: A multi-task benchmark and analysis platform for natural language understanding},
  year      = {2018},
  pages     = {353--355},
}

@Article{WinoGrande,
  author    = {Sakaguchi, Keisuke and Bras, Ronan Le and Bhagavatula, Chandra and Choi, Yejin},
  journal   = {Communications of the ACM},
  title     = {Winogrande: An adversarial winograd schema challenge at scale},
  year      = {2021},
  number    = {9},
  pages     = {99--106},
  volume    = {64},
  publisher = {ACM New York, NY, USA},
}

@Article{NVsparse,
  author  = {Mishra, Asit and Latorre, Jorge Albericio and Pool, Jeff and Stosic, Darko and Stosic, Dusan and Venkatesh, Ganesh and Yu, Chong and Micikevicius, Paulius},
  journal = {arXiv preprint arXiv:2104.08378},
  title   = {Accelerating sparse deep neural networks},
  year    = {2021},
}

@Article{Llama2,
  author  = {Touvron, Hugo and Martin, Louis and Stone, Kevin and Albert, Peter and Almahairi, Amjad and Babaei, Yasmine and Bashlykov, Nikolay and Batra, Soumya and Bhargava, Prajjwal and Bhosale, Shruti and others},
  journal = {arXiv preprint arXiv:2307.09288},
  title   = {Llama 2: Open foundation and fine-tuned chat models},
  year    = {2023},
}

@Article{Llama3,
  author  = {Grattafiori, Aaron and Dubey, Abhimanyu and Jauhri, Abhinav and Pandey, Abhinav and Kadian, Abhishek and Al-Dahle, Ahmad and Letman, Aiesha and Mathur, Akhil and Schelten, Alan and Vaughan, Alex and others},
  journal = {arXiv preprint arXiv:2407.21783},
  title   = {The llama 3 herd of models},
  year    = {2024},
}
\end{document}